\documentclass[lettersize,journal]{IEEEtran}
\usepackage{amsmath,amsfonts, amsthm}
\usepackage{array}
\usepackage{textcomp}
\usepackage{stfloats}
\usepackage{url}
\usepackage{verbatim}
\usepackage{graphicx}
\usepackage{cite}

\usepackage{subfigure}
\usepackage[linesnumbered, ruled, vlined]{algorithm2e}
\usepackage{algpseudocode}
\usepackage{array}
\usepackage{caption, subcaption}
\usepackage{textcomp}
\usepackage{stfloats}
\usepackage{url}
\usepackage{verbatim}
\usepackage{graphicx}
\usepackage{cite}
\usepackage{amssymb}
\usepackage{bm}
\usepackage{color}
\usepackage{mathrsfs}
\usepackage{hyperref}
\usepackage{multirow}
\usepackage{booktabs}
\usepackage{makecell}
\usepackage{tabularx}
\usepackage{threeparttable}
\usepackage{enumerate}
\usepackage{balance}
\usepackage{enumerate}

\theoremstyle{definition}

\newtheorem{corollary}{Corollary}

\newtheorem{lemma}{Lemma}

\begin{document}

% \title{Priority-Graph Conditioned Stackelberg Decision Making \\ for Decentralized Multi-Agent Reinforcement Learning \\ in Interactive Driving}

% \title{TSC: Topological Stackelberg Coordination via Weaving Priority Graphs \\ for Multi-Agent Reinforcement Learning in Interactive Driving}

\title{TSC: Topology-Conditioned Stackelberg Coordination for Multi-Agent Reinforcement Learning in Interactive Driving}

\author{Xiaotong Zhang,
        Gang Xiong,~\IEEEmembership{Senior Member,~IEEE,}
        Yuanjing Wang,
        Siyu Teng,~\IEEEmembership{Student Member,~IEEE,}
        % Lingxi Li,~\IEEEmembership{Senior Member,~IEEE,}
        Alois Knoll,~\IEEEmembership{Fellow,~IEEE,}
        Long Chen,~\IEEEmembership{Senior Member,~IEEE}

\thanks{
This work was supported by the National Natural Science Foundation of China under Grant 62373356 and U1909204, the National Key R\&D Program of China under Grant 2022YFB4703700 and Key-Area Research and Development Program of Guangdong Province (2020B0909050001). 
\textit{(Xiaotong Zhang and Gang Xiong are co-first authors.)} \textit{(Corresponding author: Long Chen.)}}

\thanks{Xiaotong Zhang is with the State Key Laboratory of Multimodal Artificial Intelligence Systems, Institute of Automation, Chinese Academy of Sciences, Beijing 100190, China, and also with School of Artificial Intelligence, University of Chinese Academy of Sciences, Beijing 100049, China (e-mail: zhangxiaotong2023@ia.ac.cn).}

\thanks{Gang Xiong is with the State Key Laboratory of Multimodal Artificial Intelligence Systems, Institute of Automation, Chinese Academy of Sciences, Beijing 100190, China (e-mail: gang.xiong@ia.ac.cn).}

\thanks{Yuanjing Wang is with the Department of Natural Sciences, University of Durham, Durham DH1 3LE, United Kingdom (e-mail: yuanjingwang23@gmail.com).}%

\thanks{Siyu Teng is with the College of Civil and Transportation Engineering, Shenzhen University, Guangdong, 518060, China (e-mail: siyuteng@ieee.org).}

% \thanks{Lingxi Li is with Elmore Family School of Electrical and Computer Engineering, College of Engineering, Purdue University, Indianapolis, IN 46202, USA (e-mail: lingxili@purdue.edu).}

\thanks{Alois Knoll is with the Chair of Robotics, Artificial Intelligence and Realtime Systems, Technical University of Munich, 80333 Munich, Germany (e-mail: k@tum.de).}

\thanks{Long Chen is with the State Key Laboratory of Management and Control for Complex Systems, Institute of Automation, Chinese Academy of Sciences, Beijing, 100190, China, and also with Waytous Inc., Qingdao 266109, China (e-mail: long.chen@ia.ac.cn).}}


\markboth{Under Review}%
{Shell \MakeLowercase{\textit{et al.}}: A Sample Article Using IEEEtran.cls for IEEE Journals}


\maketitle

\begin{abstract}
Safe and efficient autonomous driving in dense traffic is fundamentally a decentralized multi-agent coordination problem, where interactions at conflict points such as merging and weaving must be resolved reliably under partial observability. With only local and incomplete cues, interaction patterns can change rapidly, often causing unstable behaviors such as oscillatory yielding or unsafe commitments.
Existing multi-agent reinforcement learning (MARL) approaches either adopt synchronous decision-making, which exacerbate non-stationarity, or depend on centralized sequencing mechanisms that scale poorly as traffic density increases. To address these limitations, we propose Topology-conditioned Stackelberg Coordination (TSC), a learning framework for decentralized interactive driving under communication-free execution, which extracts a time-varying directed priority graph from braid-inspired weaving relations between trajectories, thereby defining local leader–follower dependencies without constructing a global order of play. 
Conditioned on this graph, TSC endogenously factorizes dense interactions into graph-local Stackelberg subgames and, under centralized training and decentralized execution (CTDE), learns a sequential coordination policy that anticipates leaders via action prediction and trains followers through action-conditioned value learning to approximate local best responses, improving training stability and safety in dense traffic.
Experiments across four dense traffic scenarios show that TSC achieves superior performance over representative MARL baselines across key metrics, most notably reducing collisions while maintaining competitive traffic efficiency and control smoothness. 
% Noticeably, the learned coordination induces a clear and interpretable yielding-and-passing sequence, which is further validated in real-world trials under communication-free decentralized execution.
\end{abstract}

% \textbf{Note to Practitioners}: This paper is motivated by the challenge of deploying autonomous vehicles and mobile robots in dense, unstructured environments where traffic rules are ambiguous, such as crowded intersections or busy warehouses. Current industrial solutions often rely on conservative safety rules that lead to system deadlocks or require expensive communication infrastructure to negotiate right-of-way.
% This paper proposes a decentralized coordination framework that allows agents to negotiate passage sequences solely based on onboard sensor observations, without explicit communication. By analyzing the geometric ``weaving" patterns of trajectories, the system dynamically assigns ``leader" and ``follower" roles to conflicting agents. This enables robots to predict when to yield and when to proceed, similar to how human drivers naturally negotiate a merge.
% The primary value of this approach is that it significantly improves system efficiency and resolves deadlocks without a central server. It is particularly applicable to fleet management for autonomous taxis and automated guided vehicles (AGVs) where bandwidth is limited. Future work will address robust decision-making under perception uncertainty and mixed-autonomy interactions.


\section{Introduction}

\IEEEPARstart{A}{utonomous} vehicles (AVs) must coordinate safely and efficiently in dense traffic, which is a key capability for next-generation intelligent transportation systems \cite{Wang2024Longitudinal, khalil2022connected, liu2025learning, al2025autonomous, meng2020eco, wang2024collaborative, liang2020leader}. 
In contrast to single-vehicle planning, multi-vehicle coordination must satisfy safety and kinematic constraints while forming consistent right-of-way patterns in shared spaces such as intersections, merges, and on-ramps \cite{Chen2022MARLBased, Xu2024SigmaRL, Ji2025MARL, Aksjonov2021RuleBased, Li2022GameTheoretic, Liu2023Cooperative, Shu2025Decision, sun2022interactive, bal2025cooperative}. 
Recently, MARL has emerged as a practical paradigm for modeling such interactions, enabling distributed agents to learn coordination strategies through trial and error across diverse transportation tasks \cite{Lowe2017MultiAgent, Rashid2018QMIX, Foerster2018Counterfactual, Chen2022MARLBased, Feng2025RightofWay}. 
This paradigm is particularly attractive under decentralized execution, where accurate system dynamics are difficult to obtain, communication between vehicles and infrastructure is constrained, and relying on long-horizon centralized global information is often impractical \cite{Zheng2024Safe, Ding2023Provably, Xu2024XPMARL, Liu2025Safe, Lin2025Robust}.  
Nevertheless, learning stable right-of-way patterns from local observations remains challenging in dense conflict regions, where multiple vehicles compete for the same spatiotemporal gaps and must implicitly negotiate a passage order without communication. Reciprocal reaction loops and concurrent policy updates then render the effective game highly non-stationary, often manifesting as oscillatory yielding, deadlocks, or unsafe commitments \cite{Zhang2023Inducing, Zhang2024Sequential}.

From a game-theoretic perspective, decentralized coordination can be cast as a partially observed dynamic non-cooperative game, where agents choose control actions to optimize safety–efficiency objectives under coupled collision-avoidance constraints, and stable behavior corresponds to an equilibrium that implicitly defines a consistent right-of-way ordering \cite{Li2022GameTheoretic, Shu2025Decision, Zhu2022BiHierarchical}. According to decision timing, existing MARL-based coordination methods can be broadly grouped into two classes.
The first class formulates coordination as a simultaneous-move (joint-action) Markov game, where all agents select actions synchronously at each time step \cite{Chen2022MARLBased, Xu2024SigmaRL, Ji2025MARL, Wang2024Longitudinal}. Since agents update their policies concurrently during learning, the induced environment becomes non-stationary from any individual agent's viewpoint, which can destabilize training and hinder convergence \cite{Xu2024XPMARL, Lin2025Robust, Lowe2017MultiAgent, Rashid2018QMIX, Foerster2018Counterfactual}. 
The second class introduces sequential decision structures or priority mechanisms to impose an explicit or implicit order of play, approximating Stackelberg-style leader–follower interactions in which some agents commit earlier and others respond, thereby encouraging more consistent and stable coordination patterns \cite{Hu2024Who, Zhang2023Inducing, Zhang2024Sequential, Zhang2025Coupled, Li2022GameTheoretic, Liu2023Cooperative, Zhu2022BiHierarchical}. However, many such approaches depend on constructing a global order of play or unrolling a fixed multi-agent decision sequence, causing representational and computational costs to grow quickly with the number of vehicles and limiting scalability in large, heterogeneous dense traffic settings.

\begin{figure*}[t]
    \centering
    \includegraphics[width=6.5in]{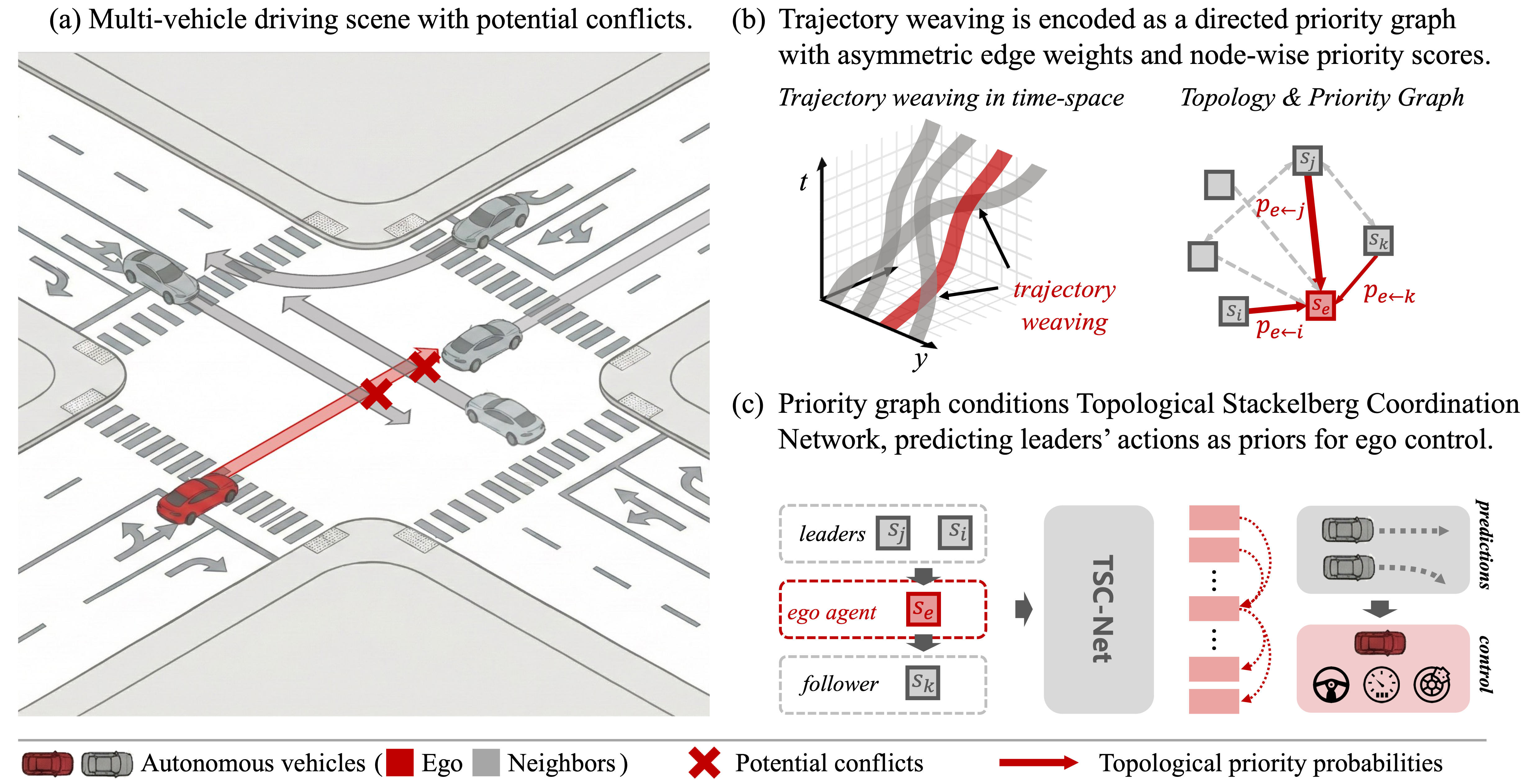}
    \setlength{\belowcaptionskip}{-0.5cm}
    \caption{Topology-conditioned Stackelberg coordination from trajectory weaving in multi-vehicle traffic.
    }
    \label{fig:intro}
\end{figure*}

Topological structures have been widely used to describe multi-agent interaction patterns that arise from trajectory weaving, merging, and overtaking \cite{Mavrogiannis2023Abstracting, Liu2024Reasoning}. 
% In particular, braid-theoretic formulations abstract crossings and precedence relations into topological braids or directed graphs, enabling compact reasoning about interaction complexity, homotopy classes, and priority/deadlock constraints in multi-robot coordination \cite{Mavrogiannis2023Abstracting}. 
In particular, braid-theoretic formulations abstract crossings and precedence relations into topological braids or directed graphs, where the crossing order of intertwined trajectories induces a directed precedence structure that supports compact reasoning about interaction complexity and priority constraints in multi-robot coordination \cite{Mavrogiannis2023Abstracting}.
% In autonomous driving, topology has been used to structure multi agent interaction prediction and planning, for example by discretizing joint futures into topological equivalence classes or by leveraging lane connectivity and homotopy cues to obtain more consistent and rule compliant behaviors \cite{Liu2024Reasoning, xiao2025srefiner, Mascetta2024RuleCompliant, Chen2025Survey}. 
In autonomous driving, related braid-inspired abstractions have been adopted to organize multi-vehicle interaction prediction and planning, for example by grouping joint futures into topological equivalence classes based on relative crossing order, or by leveraging lane connectivity and homotopy cues to obtain more consistent and rule compliant behaviors \cite{Liu2024Reasoning, xiao2025srefiner, Mascetta2024RuleCompliant, Chen2025Survey}.
Despite these advances, topology is typically introduced as an external representation \cite{Mavrogiannis2023Abstracting, Liu2024Reasoning}, rather than being embedded into the underlying multi-agent decision structure to induce feasible orders of play and directly guide policy learning under communication constrained decentralized execution \cite{Zheng2024Safe, Ding2023Provably, Xu2024XPMARL, Liu2025Safe}.

To address these limitations, we propose \underline{T}opology-conditioned \underline{S}tackelberg \underline{C}oordination (TSC), a multi-agent learning framework for safe and efficient cooperative interactive driving in dense traffic under communication-free execution. In shared conflict regions, interactions are governed by an evolving local precedence structure, whereas enforcing a global order of play is brittle and scales poorly with increasing traffic density \cite{Hu2024Who, Zhang2023Inducing, Zhang2024Sequential, Zhang2025Coupled, Li2022GameTheoretic, Shu2025Decision, Zhu2022BiHierarchical}. Instead, TSC infers a time-varying directed priority graph from braid-inspired trajectory weaving cues, which reflect prospective crossing order and the strength of interaction coupling from ego-centric observations. 
This graph provides a compact interaction scaffold. Using this graph, each agent selects a limited number of interaction-critical neighbors to form a sparse decision state, and the resulting directed relations define local leader-follower dependencies that support graph-local Stackelberg reasoning rather than a centralized decision sequence.
By coupling priority-guided sparsification with local sequential reasoning, TSC reduces training non-stationarity in dense interactions and concentrates learning on safety-critical conflicts. To implement TSC, we propose a topology-conditioned priority inference and policy learning network (TSC-Net) trained under centralized training and decentralized execution. TSC-Net jointly learns priority graph inference and topology-conditioned policy, with consistency-driven supervision that links edge-wise priority probabilities with score-induced precedence so that inferred priorities remain stable and decision-relevant. At execution time, each agent acts only from compact priority-filtered local observations, enabling communication-free yet coherent leader-follower behavior.

\noindent\textbf{Contributions.} We summarize our contributions as follows:
\begin{itemize}
    \item We propose TSC, a Topology-conditioned Stackelberg Coordination framework for dense multi-vehicle interaction under communication-free decentralized execution, which infers a time-varying directed priority graph from braid-inspired weaving cues to induce graph-local leader-follower dependencies without constructing a global order of play.
    \item We develop TSC-Net, a CTDE actor-critic instantiation of TSC that jointly learns edge-wise priority probabilities and node-wise priority scores with a consistency regularizer, and conditions value learning on predicted leader actions to enable stable local sequential reasoning without explicit communication.
    \item We evaluate TSC in four dense traffic scenarios, showing consistent collision reduction and improved interaction stability, with ablations that validate the impact of priority inference, topology-guided sparsification, and local Stackelberg reasoning.
\end{itemize}

% \vspace{-0.3cm}

% \vspace{-0.5cm}
\section{Problem Formulation}
\label{sec:problem}

\subsection{Task and Objective}
\label{subsec:problem}

We consider an interactive driving scenario with $N\ge 2$ connected and
automated vehicles evolving in discrete time
\cite{He2022IPLAN,Xu2024SigmaRL,Martinez2025AVOCADO}. Each vehicle is controlled
by an agent and we use the same index $i$ for both. Let
$\mathcal{I}=\{1,\dots,N\}$ and $\mathcal{T}=\{1,\dots,T\}$ denote the agent set
and episode time indices.

We model the problem as a partially observable stochastic game with tuple:
\begin{equation}
\label{eq:tuple}
\left\langle
N,\mathcal{S},\{\mathcal{A}_i\}_{i=1}^N,\{\mathcal{O}_i\}_{i=1}^N,\{\Omega_i\}_{i=1}^N,
\mathcal{P},\{r_i\}_{i=1}^N,T,\gamma
\right\rangle,
\end{equation}
where $\mathcal{S}\subseteq \mathbb{R}^{d_s}$ is the global state space and $\mathcal{A}_i=\mathcal{A}\subseteq \mathbb{R}^{d_a}$.
At time step $t$, the system is at a global state $s^t\in\mathcal{S}$.
Each agent $i$ chooses an action $a_i^t\in\mathcal{A}_i$ and receives a local observation
$o_i^t\in\mathcal{O}_i$ given by $o_i^t = \Omega_i(s^t)$, where $\Omega_i$ maps the global state to the observable information of agent $i$.
The stochastic transition function is denoted by $\mathcal{P}:\mathcal{S}\times \mathcal{A}_1\times\cdots\times \mathcal{A}_N\rightarrow \Delta(\mathcal{S})$,
and the next state is sampled as $s^{t+1} \sim \mathcal{P}\big(s^t, a_1^t, \dots, a_N^t\big)$.
Agent $i$ receives a scalar reward $r_i^t = r_i\big(s^t, a_1^t, \dots, a_N^t\big)$, and $\gamma\in(0,1]$ is the discount factor.

\textbf{Local observation.}
At time $t$, the set of agents observable to agent $i$ is denoted by $\mathcal{N}_i^t \subseteq \{1,\dots,N\}\setminus\{i\}$.
For implementation, $\mathcal{N}_i^t$ is represented as a fixed-size candidate set of capacity $M\in\mathbb{N}$, padding when fewer are available.
The local observation is represented as
$o_i^t = \{\tilde{o}_i^t,\{o_{i,j}^t\}_{j\in\mathcal{N}_i^t}\}$,
where $\tilde{o}_i^t$ collects ego-related features of agent $i$ and $o_{i,j}^t$ encodes features of neighbor $j$ as seen from agent $i$. This can be obtained from on-board sensing without inter-agent communication.

\textbf{Policy and objective.}
Each agent follows a decentralized policy
$\pi_\theta:o_i\rightarrow\Delta(\mathcal{A}_i)$ that maps its local observation to an action distribution. Under CTDE with parameter sharing
\cite{Lowe2017MULTI,Rashid2020MONOTONIC}, we optimize shared parameters $\theta$
to maximize the average discounted return
\begin{equation}
\label{eq:objective}
\theta^{*}
=
\arg\max_{\theta}
\ \mathbb{E}_{\pi_\theta}\!\left[
\sum_{t=1}^{T}
\gamma^{t}\, \frac{1}{N}\sum_{i=1}^{N} r_i\!\left(s^{t}, a_1^{t}, \dots, a_N^{t}\right)
\right].
\end{equation}
This setting is challenging due to non-stationarity induced by mutual reactions
and the need for consistent implicit ordering in dense conflict regions.

\subsection{Topological Interaction Graph}
\label{subsec:topo_graph}

In dense interactions, purely distance-based neighborhoods may be insufficient to capture behaviorally coupled
relations such as merging or crossing. We therefore introduce a directed
topological interaction graph that models local precedence and influence as directed dependencies among agents
and supports priority-aware coordination under decentralized execution.
At time $t$, the scene is represented by a directed graph
\begin{equation}
G_t = (V_t, E_t, W_t),
\end{equation}
where each node $v_i^t\in V_t$ corresponds to agent $i$ with node feature
$x_i^t\in\mathbb{R}^{d_x}$. For $j\neq i$, an edge $(j,i)\in E_t$ indicates that
agent $j$ can constrain the feasible maneuvers of agent $i$ and carries a
nonnegative weight $w_{i\leftarrow j}^t\in\mathbb{R}{\ge 0}$ that is not
necessarily symmetric and will be instantiated as a directed priority strength in our construction. To respect partial observability, edge features
$z{i\leftarrow j}^t\in\mathbb{R}^{d_z}$ are computed from quantities available
to agent $i$ via on-board sensing. Under decentralized execution, agent $i$
constructs an ego-centric local view $G_i^t$ induced by
${i}\cup\mathcal{N}_i^t$ and performs subsequent inference and decision making
from $G_i^t$ using only this local view without global scene access or communication \cite{Liu2024Reasoning}.

% \vspace{-0.3cm}

\section{Method}

Fig.~\ref{fig:intro} summarizes the proposed framework and its instantiation in TSC-Net. Sec.~\ref{subsec:formulation} formalizes the trajectory weaving-based priority formulation and constructs a directed priority graph with asymmetric edge weights that encodes local precedence among potentially conflicting vehicles. Sec.~\ref{subsec:TSC-Net} then presents TSC-Net, which infers an ego-centric local priority graph and performs topology-conditioned Stackelberg coordination using predicted leader actions to support communication-free execution. Sec.~\ref{subsec:training} describes the learning objective and the CTDE training procedure.

\vspace{-0.5cm}

\subsection{Topological Priority Formulation}
\label{subsec:formulation}

\subsubsection{Agent-centric future-trajectory representation}
\label{subsubsec:trajectory-representation}

Let $\mathcal{Y}^t=\{\mathbf{Y}_i^t\}_{i=1}^{N}$ denote the multi-agent future trajectories over a finite horizon $H$, where
$\mathbf{Y}_i^t=\{\mathbf{X}_i^{t+h}\}_{h=0}^{H}\in\mathbb{R}^{(H+1)\times 2}$
and $\mathbf{X}_i^{t+h}=\big[p_{x,i}^{t+h},\,p_{y,i}^{t+h},]^\top \in\mathbb{R}^{2}$
collects 2D position.
These future rollouts are used only for training supervision and are unavailable at decentralized execution.
We are interested in how these trajectories intertwine in the neighborhood of each agent and in representing such interactions in a structured form that can be exploited by a graph-based decision model.

For each ego agent $i$, we introduce an agent-centric lateral-longitudinal coordinate frame.
Let $\theta_i^t\in\mathbb{R}$ denote its heading over the horizon and
$R_i^t\in\mathbb{R}^{2\times 2}$ the rotation matrix that aligns the $x$-axis of this frame with $\theta_i^t$.
The future position of any agent $j$ is mapped into this local frame over
a horizon $H$ as:
\begin{equation}
    \tilde{\mathbf{X}}_j^{(i)}(t+h)=R_i\bigl(\mathbf{X}_j^{t+h}-\mathbf{X}_i^{t}\bigr)\in\mathbb{R}^{2},
    \quad h = 0,\dots,H
\end{equation}

We denote its lateral coordinate by $y_k^{(n)}(t+h)$, and $\{y_j^{(i)}(t+h)\}_{h=0}^{H}\in\mathbb{R}^{H+1}$ describes, from the viewpoint of agent $i$, the sideways motion of $j$ relative to $i$.

\subsubsection{Weaving distance and soft pairwise topological priority}
\label{subsubsec:topological-priority}

Consider two distinct agents $(i,j)$ at time $t$. In the ego-centric
lateral-longitudinal frame of agent $i$, define the lateral gap as:
\begin{equation}
    \Delta^{(i)}_{i\leftarrow j}(h)
    = y_i^{(i)}(t+h) - y_j^{(i)}(t+h),
    \quad h = 0,\dots,H.
\end{equation}

Intuitively, strong weaving occurs when the two future lateral profiles
approach each other and tend toward a braid crossing, i.e., $\Delta^{(i)}_{i\leftarrow j}(h)$
becomes small in magnitude and exhibits a sign change over
consecutive steps.
This crossing event induces a local crossing-order cue, which we use to define directed precedence between the pair and to supervise the corresponding edge in the priority graph.

To quantify this effect, for $h=0,\dots,H-1$ we define a per-step near-crossing
score:
\begin{equation}
    \delta^{(i)}_{i\leftarrow j}(h)
    =
    \frac{
        \min\bigl(|\Delta^{(i)}_{i\leftarrow j}(h)|,\ |\Delta^{(i)}_{i\leftarrow j}(h+1)|\bigr)
    }{
        \varepsilon + \max\bigl(0,\,-\Delta^{(i)}_{i\leftarrow j}(h)\,\Delta^{(i)}_{i\leftarrow j}(h+1)\bigr)
    }
    \in\mathbb{R}_{+},
\end{equation}
where $\varepsilon>0$ is a stabilizing constant. The numerator measures the
closest lateral separation over two consecutive steps, while the denominator
downweights non-interleaving cases and yields smaller values when a sign change
is likely (since $-\Delta(h)\Delta(h+1)>0$).
We then aggregate over the horizon by taking the most critical (closest)
interleaving event, defining the directed weaving distance:
\begin{equation}
    d^t_{i\leftarrow j}
    =
    \min_{h=0,\dots,H-1}\delta^{(i)}_{i\leftarrow j}(h).
\end{equation}

Smaller $d^t_{i\leftarrow j}$ indicates stronger topological coupling between
the future motions of $i$ and $j$ as viewed from agent $i$'s frame, whereas larger
values indicate weak or negligible weaving. Repeating the same construction
with roles swapped (i.e., in the local frame of agent $j$) yields $d^t_{j\leftarrow i}$.

Given the two directed distances, we assign asymmetric pairwise priority
probabilities via a two-class softmax:
\begin{equation}
    p^t_{i\leftarrow j}
    =
    \frac{\exp(-d^t_{i\leftarrow j}/\tau)}
         {\exp(-d^t_{i\leftarrow j}/\tau)+\exp(-d^t_{j\leftarrow i}/\tau)},
    \quad
    p^t_{j\leftarrow i} = 1 - p^t_{i\leftarrow j},
\end{equation}
where $\tau>0$ controls the softness. With the arrow convention $i\leftarrow j$
(meaning influence from $j$ to $i$), a larger $p^t_{i\leftarrow j}$ indicates
that $j$ is more dominant relative to $i$ within their intertwined region.
Collecting agents as vertices $V_t$, the pairwise probabilities
$\{p^t_{i\leftarrow j}\}$ define the edge weights of the topological interaction
graph $G_t=(V_t,E_t,W_t)$ introduced in Sec.~\ref{subsec:topo_graph}, by setting
$w^t_{i\leftarrow j}=p^t_{i\leftarrow j}$ for each interaction relevant ordered
pair $(j,i)\in E_t$. The general asymmetry $p^t_{i\leftarrow j}\neq p^t_{j\leftarrow i}$
thus encodes directional influence on $G_t$.

\subsubsection{Scalar priority field from pairwise probabilities}
\label{subsubsec:scalar-priority}

The pairwise probabilities $\{p^t_{i\leftarrow j}\}$ specify edge-level
preferences between interacting agents, but do not directly yield a compact
ordering over the whole interaction graph.
We therefore assign each agent $i$ at time $t$ a scalar topological priority
score $s_i^t\in\mathbb{R}$, where a larger value indicates that $i$ tends to
dominate (i.e., be yielded to by) its neighbors.
The collection $\{s_i^t\}$ can be viewed as a scalar potential on the nodes of
$G_t$ whose pairwise differences summarize directed priority relations.

To derive supervision from these probabilities, we define an
antisymmetric preference signal for each ordered pair:
\begin{equation}
A^t_{i\leftarrow j}
    =
    p^t_{j\leftarrow i} - p^t_{i\leftarrow j},
    \qquad
    A^t_{j\leftarrow i} = -A^t_{i\leftarrow j}.
\end{equation}

With this convention, $A^t_{i\leftarrow j}>0$ indicates that $i$ should be
assigned higher priority than $j$.
We then recover node scores by fitting score differences to these signed
preferences via a weighted least-squares objective:
\begin{equation}
    \{s_i^{t,\star}\}
    =
    \arg\min_{\{s_i\}}
    \frac{1}{2}\sum_{(j,i)\in E_t} c^t_{i\leftarrow j}
    \bigl( (s_i - s_j) - A^t_{i\leftarrow j} \bigr)^2,
\end{equation}
subject to the gauge constraint $\sum_{i=1}^{N} s_i = 0$, which removes the
global shift ambiguity of $\{s_i\}$.
Here the summation is taken over interaction-relevant directed edges
$(j,i)\in E_t$, and $c^t_{i\leftarrow j}$ weights the confidence of the
corresponding preference, downweighting uncertain pairs (e.g., when
$p^t_{i\leftarrow j}\approx\tfrac{1}{2}$), and we set $c^t_{i\leftarrow j}=\bigl|p^t_{i\leftarrow j}-\tfrac{1}{2}\bigr|^\alpha$.
Solving this problem yields, for each time $t$, a scalar priority field
$\{s_i^{t,\star}\}$ that provides a graph-level ordering consistent (in the
least-squares sense) with the edge-wise preferences.

\begin{figure*}[t]
    \centering
    \setlength{\belowcaptionskip}{-0.5cm}
    \includegraphics[width=6.5in]{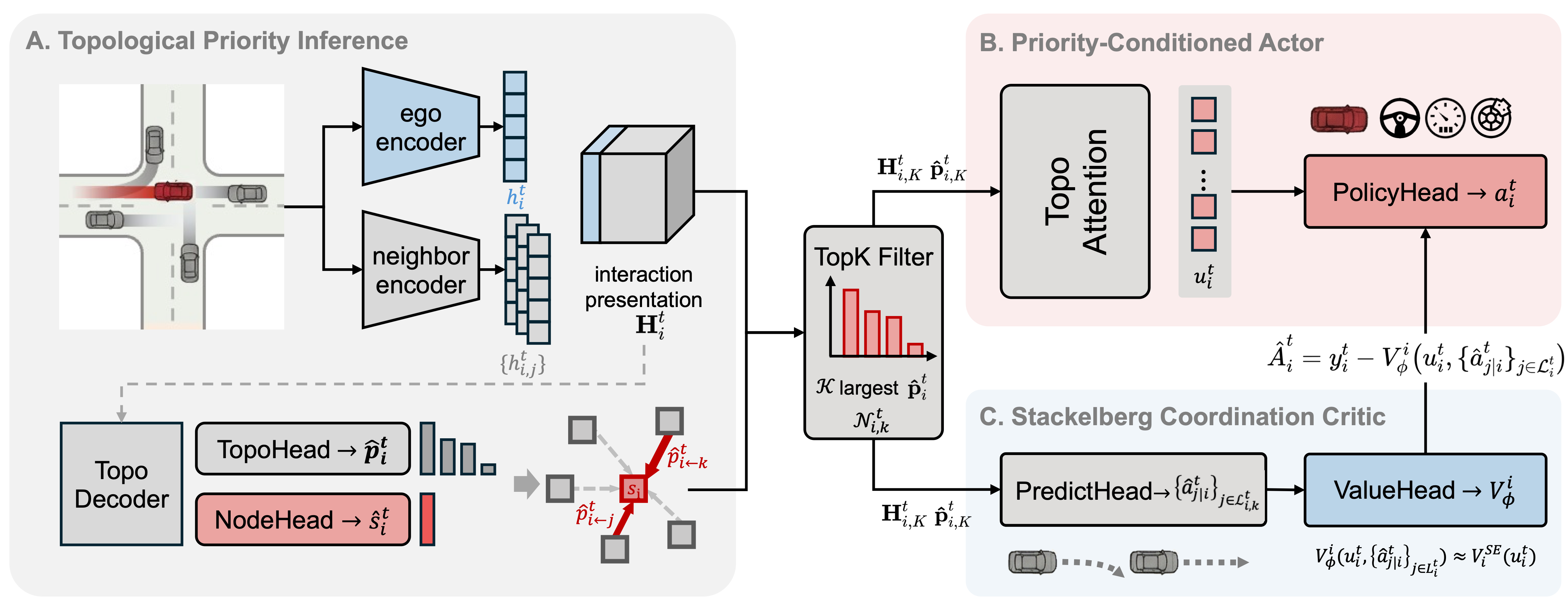}
    \caption{\textbf{The TSC-Net architecture.}
    Ego and neighbor encoders construct an interaction representation $\mathbf{H}_i^t$, from which the TopoDecoder predicts pairwise priorities $\mathbf{\hat{p}}_{i\leftarrow j}^t$ and node scores $\mathbf{\hat{s}}_i^t$ to form a local priority graph. Guided by this graph, a TopK filter and TopoAttention aggregate the most influential neighbors into a compact decision state $u_i^t$, which the PolicyHead maps to the continuous action $a_i^t$. The same priority structure defines the local leader set $\mathcal{L}_i^t$. The PredictHead estimates leader actions $\{\hat{a}_{j\mid i}^t\}_{j\in\mathcal{L}_i^t}$, and the ValueHead conditions on these predictions to approximate the local Stackelberg value, yielding TD advantages $\hat{A}_i^t$ for actor optimization under CTDE.
    }
  \label{fig:tscnet}
\end{figure*}

\subsubsection{Joint prediction and supervision of pairwise and scalar priorities}
\label{subsubsec:prediction-supervision-priorities}

At each time step $t$, agent $i$ predicts edge-wise priority probabilities $\hat{p}^t_{i\leftarrow j}\in[0,1]$ for each neighbour $j\in\mathcal{N}_i^t$ and a node-wise scalar priority score $\hat{s}_i^t\in\mathbb{R}$ for itself from its ego-centric observation $o_i^t$.
The training objective combines three terms. An edge-level loss fits the  priorities,
\begin{equation}
    \mathcal{L}_{\mathrm{edge}}
    =
    \sum_t \sum_{(j,i)\in E_t}
    \mathrm{BCE}\bigl(\hat{p}^t_{i\leftarrow j},\, p^t_{i\leftarrow j}\bigr),
\end{equation}
a node-level loss aligns the predicted scalar field with the least-squares solution,
\begin{equation}
    \mathcal{L}_{\mathrm{node}}
    =
    \sum_t \sum_i
    \bigl\|\hat{s}_i^t - s_i^{t,\star}\bigr\|_2^2,
\end{equation}
and a consistency term encourages the edge probabilities to agree with the ordering induced by the node scores.
Specifically, we construct
\begin{equation}
    \tilde{p}^t_{i\leftarrow j}
    =
    \sigma\!\left(\frac{\hat{s}_j^t - \hat{s}_i^t}{\tau_s}\right),
\end{equation}
with temperature $\tau_s>0$ and logistic sigmoid $\sigma(\cdot)$, and penalise their discrepancy via
\begin{equation}
    \mathcal{L}_{\mathrm{cons}}
    =
    \sum_t \sum_{(j,i)\in E_t}
    \bigl(\hat{p}^t_{i\leftarrow j} - \tilde{p}^t_{i\leftarrow j}\bigr)^2.
\end{equation}

The overall topological loss is:
\begin{equation}
    \mathcal{L}_{\mathrm{topo}}
    =
    \mathcal{L}_{\mathrm{edge}}
    +
    \lambda_{\mathrm{node}}\,\mathcal{L}_{\mathrm{node}}
    +
    \lambda_{\mathrm{cons}}\,\mathcal{L}_{\mathrm{cons}},
\end{equation}
with trade-off weights $\lambda_{\mathrm{node}},\lambda_{\mathrm{cons}}>0$.
At test time, each agent only requires its own observation $o_i^t$ to compute $\hat{s}_i^t$ and $\hat{p}^t_{i\leftarrow j}$, enabling decentralized estimation of both pairwise and scalar priorities without any central coordinator or inter-vehicle communication.

\subsection{Topology-conditioned Stackelberg Coordination Network (TSC-Net)}
\label{subsec:TSC-Net}

As shown in Fig.~\ref{fig:tscnet}, we present the Topology-conditioned Stackelberg Coordination Network (TSC-Net), which instantiates TSC under centralized training and decentralized execution. Guided by the topological priority formulation in Sec.~\ref{subsec:formulation}, TSC-Net infers an ego-centric local priority graph from observations and uses it to enable topology-conditioned leader–follower coordination under decentralized, communication-free execution.

\subsubsection{Trajectory weaving for local priority graph inference}
\label{subsubsec:topological-behavioral}

Given the ego-centric observation $o_i^t$ defined in Sec.~\ref{subsec:problem}, agent $i$ constructs an agent-centric representation for local priority graph inference.
We encode the ego part and each neighbor part separately to obtain latent features:
\begin{equation}
    h_i^t = \mathrm{Encoder}_{\mathrm{ego}}(\tilde{o}_i^t),\quad
    h_{i,j}^t = \mathrm{Encoder}_{\mathrm{nbr}}(o_{i,j}^t),
    \ j\in\mathcal{N}_i^t,
\end{equation}
where $h_i^t\in\mathbb{R}^{d_e}$ is the ego embedding, $h_{i,j}^t\in\mathbb{R}^{d_n}$ encodes the interaction from neighbor $j$ as observed by agent $i$.
Stacking the ego and neighbor embeddings yields an ego-centric interaction representation:
\begin{equation}
    \mathbf{H}_i^t
    =
    \bigl[h_i^t;\ \{h_{i,j}^t\}_{j\in\mathcal{N}_i^t}\bigr]
    \in \mathbb{R}^{d_e + M d_n},
\end{equation}
where $\{h_{i,j}^t\}_{j\in\mathcal{N}_i^t}$ denotes the collection of neighbour embeddings defined over $\mathcal{N}_i^t$ with padding when fewer neighbors are available.

On top of $\mathbf{H}_i^t$ we introduce a topological priority decoding stream that recovers edge-wise and node-wise priorities induced by trajectory weaving:
\begin{align}
    \mathbf{Q}_i^{\mathrm{topo}}
        &= \mathrm{TopoDecoder}(\mathbf{H}_i^t),
        \label{eq:topodecoder}\\
    \hat{\mathbf{p}}_i^t
        &= \mathrm{TopoHead}(\mathbf{Q}_i^{\mathrm{topo}}),
        \label{eq:topohead}\\
    \hat{s}_i^t
        &= \mathrm{NodeHead}(\mathbf{Q}_i^{\mathrm{topo}}).
        \label{eq:nodehead}
\end{align}

Specifically, the agent-centric representation $\mathbf{H}_i^t$ is first mapped to per-neighbor topology features $\mathbf{Q}_i^{\mathrm{topo}}\in\mathbb{R}^{M\times d_t}$.
These features are then fed into two lightweight heads to produce edge-wise priority probabilities $\hat{\mathbf{p}}_i^t\in[0,1]^M$ and a node-wise scalar priority score $\hat{s}_i^t\in\mathbb{R}$.

During training, the priority inference branch is supervised by the
 topological priority labels $p_{i\leftarrow j}^t$ and $s_i^{t,\star}$.
The edge labels satisfy the reciprocity constraint
\begin{equation}
    p_{i\leftarrow j}^t + p_{j\leftarrow i}^t = 1.
\end{equation}

The edge labels are de-cycled offline to suppress short directed loops in the induced priority graph.
Consequently, each agent can infer a consistent local ordering signal from $\mathbf{H}_i^t$ under decentralized execution.

\subsubsection{Priority-Conditioned Local Decision Interface}
\label{subsubsec:decision-interface}

Before performing Stackelberg coordination, each agent compresses its raw neighborhood into a compact, priority filtered subset. The pairwise priorities $\hat{p}_{i\leftarrow j}^t$ provide a directional yielding cue under our arrow convention $i\leftarrow j$. Larger values indicate that neighbor $j$ is more dominant in the inferred weaving relation and should therefore receive greater attention when agent $i$ plans.

To keep the per-agent control input bounded while retaining the most influential
interactions, agent $i$ selects the $K$ neighbors with the largest priorities
and aggregates their latent features:
\begin{align}
    \mathcal{N}_{i,K}^t
        &=
        \mathrm{TopK}_K\bigl(\{\hat{p}_{i\leftarrow j}^t\}_{j\in\mathcal{N}_i^t}\bigr),
        \label{eq:topk}\\
    C_i^t
        &=
        \mathrm{TopoAttn}\Bigl([h_i^t;\ \{h_{i,j}^t\}_{j\in\mathcal{N}_{i,K}^t}]\Bigr)
        \in \mathbb{R}^{d_c},
        \label{eq:topoattn}
\end{align}
where $\mathcal{N}_{i,K}^t\subseteq \mathcal{N}_i^t$ denotes the selected
neighbor set. This priority-guided compression focuses the policy on strongly
coupled interactions and supports scalability as traffic density increases.
On top of the aggregated context, we define the actor branch of TSC-Net. The
context $C_i^t$ and scalar priority $\hat{s}_i^t$ are mapped to a local decision
state $u_i^t$, which parametrizes a stochastic policy:
\begin{align}
    u_i^t
        &= \mathrm{EgoDecoder}\bigl(C_i^t,\ \hat{s}_i^t\bigr),\\
    a_i^t
        &\sim \pi_\theta(\cdot \mid u_i^t)
        = \mathrm{PolicyHead}_\theta(u_i^t).
\end{align}

The $\mathrm{PolicyHead}$ outputs mean and diagonal standard deviation of a Gaussian distribution. The sampled action is then squashed by $\tanh$ and rescaled to continuous control space.

This subsection specifies the actor interface only. The critic is introduced in
Sec.~\ref{subsubsec:stackelberg-coordination}, where it conditions on $u_i^t$
together with predicted actions of priority-selected neighbors to evaluate the
priority-guided leader-follower structure, while keeping such predictions out
of the policy input for training stability.

\subsubsection{Stackelberg Coordination on the Topological Priority Graph}
\label{subsubsec:stackelberg-coordination}

We define a Stackelberg-style local decision rule on the inferred priority
graph by identifying, for each agent $i$, a subset of neighbors that are treated
as \emph{local leaders} at time $t$. Among the Top-$K$ neighbors
$\mathcal{N}_{i,K}^t$ selected in \eqref{eq:topk}, we define:
\begin{equation}
    \mathcal{L}_i^t=
    \Bigl\{
        j\in\mathcal{N}_{i,K}^t \,\Big|\,
        \hat{p}_{i\leftarrow j}^t > \tfrac{1}{2} + \Delta_p
    \Bigr\},
\end{equation}
where $\Delta_p\ge 0$ is a small margin to avoid unstable ties around $0.5$.
Under the arrow convention $i\leftarrow j$, $\mathcal{L}_i^t$ collects those
neighbors whose inferred weaving relation suggests that $i$ should yield to $j$.

Under parameter sharing, agents execute different instances of the same policy
$\pi_\theta$. However, the realized leader actions
$a_j^t\sim\pi_\theta(\cdot\mid u_j^t)$ are not observable to agent $i$ at
decision time. We therefore predict one-step neighbor actions using the
priority-filtered embeddings
$\mathbf{H}_{i,K}^t=\{h_{i,j}^t\}_{j\in\mathcal{N}_{i,K}^t}$ and priorities
$\hat{\mathbf{p}}_{i,K}^t=\{\hat{p}_{i\leftarrow j}^t\}_{j\in\mathcal{N}_{i,K}^t}$
\begin{equation}
    \hat{\mathbf{a}}_i^t
    =
    \mathrm{PredictHead}\bigl(\mathbf{H}_{i,K}^t,\ \hat{\mathbf{p}}_{i,K}^t\bigr),
\end{equation}
where $\hat{\mathbf{a}}_i^t=\{\hat{a}_{j\mid i}^t\}_{j\in\mathcal{N}_{i,K}^t}$
are the predicted neighbor actions as perceived by agent $i$. We use
$\{\hat{a}_{j\mid i}^t\}_{j\in\mathcal{L}_i^t}$ as topology-guided
approximations of the (unobserved) leader actions in the local coordination
game.

Given the ego decision state $u_i^t$ and predicted leader actions, we implement a
Stackelberg-conditioned critic via
\begin{equation}
    V_\phi^i\bigl(
        u_i^t,\ \{\hat{a}_{j\mid i}^t\}_{j\in\mathcal{L}_i^t}
    \bigr)
    =
    \mathrm{ValueHead}\Bigl(
        u_i^t,\ \{\hat{a}_{j\mid i}^t\}_{j\in\mathcal{L}_i^t}
    \Bigr),
    \label{eq:valhead}
\end{equation}
trained by temporal-difference learning to estimate the expected discounted
return of agent $i$ under the induced local hierarchy. We formalize this target
through the local Stackelberg state-value $V_i^{\mathrm{SE}}(u_i^t)$ and enforce
\begin{equation}
    V_\phi^i\bigl(u_i^t,\{\hat{a}_{j\mid i}^t\}_{j\in\mathcal{L}_i^t}\bigr)
    \approx
    V_i^{\mathrm{SE}}(u_i^t).
    \label{eq:se-approx}
\end{equation}

This construction can be interpreted through classical Stackelberg games, where
a leader commits to a policy and a follower responds optimally. While the
standard equilibrium is defined by the bilevel program:
\begin{equation}
    \pi_L^\star
    =
    \arg\max_{\pi_L}
    J_L\bigl(\pi_L,\ \mathrm{BR}(\pi_L)\bigr),
    \qquad
    \pi_F^\star = \mathrm{BR}(\pi_L^\star),
    \label{eq:classical-stackelberg}
\end{equation}

Our objective is not to solve \eqref{eq:classical-stackelberg} globally, but to
realize follower-side evaluation locally on the priority graph. For
agent $i$, treating $\mathcal{L}_i^t$ as leaders yields the Bellman recursion:
\begin{equation}
\begin{aligned}
V_i^{\mathrm{SE}}(u_i^t)
&=
\mathbb{E}\Bigl[
    r_i^t + \gamma V_i^{\mathrm{SE}}(u_i^{t+1})
    \,\Bigm|\,
    u_i^t,\ a_i^t\sim\pi_\theta(\cdot\mid u_i^t),\\
&\hspace{3.6em}
    a_j^t\sim\pi_\theta(\cdot\mid u_j^t),\ \forall j\in\mathcal{L}_i^t
\Bigr],
\end{aligned}
\label{eq:se-bellman}
\end{equation}
where neighbors in $\mathcal{L}_i^t$ act as local leaders for $i$ under parameter
sharing. Maximizing
$J_i^{\mathrm{SE}}(\pi_\theta)=\mathbb{E}[V_i^{\mathrm{SE}}(u_i^0)]$ thus
corresponds to seeking a local best response to the leader set:
\begin{equation}
    \mathrm{BR}_i(\{\pi_\theta\}_{j\in\mathcal{L}_i^t})
    \in
    \arg\max_{\pi_\theta}\,
    J_i^{\mathrm{SE}}(\pi_\theta).
    \label{eq:local-br}
\end{equation}

In decentralized execution, the leader actions in \eqref{eq:se-bellman} are
unavailable and are replaced by predictions $\hat{a}_{j\mid i}^t$. Conditioning
$V_\phi^i$ on $\{\hat{a}_{j\mid i}^t\}_{j\in\mathcal{L}_i^t}$ therefore yields a
biased but tractable approximation to the local Stackelberg backup. Under mild
assumptions on the priority graph and bounded prediction error, the induced
value bias remains controlled and the resulting advantage estimates support
stable policy-gradient updates. Lemma~\ref{lem:se-bellman-bound} and Lemma~\ref{lem:se-pdl} in Appendix provide a formal statement and
connect the induced update to the underlying local bilevel structure.

Crucially, the actor depends only on $u_i^t$ and does not take
$\{\hat{a}_{j\mid i}^t\}$ as direct inputs. Predicted leader actions shape the
value landscape through the critic, rather than being hard-coded into the
policy parameterization, which empirically improves training stability while
preserving topology-conditioned coordination.

\subsection{Learning Objective and Training Procedure}
\label{subsec:training}

We train TSC-Net under CTDE paradigm. The overall objective consists of (i) supervised learning of the
topological priority graph, (ii) neighbor-action prediction loss,
and (iii) actor-critic objective for topology-conditioned control.

The priority inference branch (Sec.~\ref{subsubsec:topological-behavioral}) predicts edge probabilities $\hat{p}{i\leftarrow j}^t$ and node scores $\hat{s}i^t$ from $o_i^t$ and is supervised by the trajectory-derived labels $p{i\leftarrow j}^t$ and $s_i^{t,\star}$. We adopt the priority inference loss $\mathcal{L}{\mathrm{topo}}$ defined in Sec.~\ref{subsubsec:topological-behavioral} to train this branch.

The prediction head outputs $\hat{a}_{j\mid i}^t$ and is trained on leaders
$j\in\mathcal{L}_i^t$ with
\begin{equation}
    \mathcal{L}_{\mathrm{lead}}
    =
    \sum_{t,i}\sum_{j\in\mathcal{L}_i^t}
    \ell_{\mathrm{act}}\bigl(\hat{a}_{j\mid i}^t,\ a_j^t\bigr),
\end{equation}
where $\ell_{\mathrm{act}}$ is a regression loss.

Given rewards $r_i^t$ and discount $\gamma$, we train the Stackelberg-conditioned
critic with a squared temporal-difference loss:
\begin{equation}
    \mathcal{L}_{\mathrm{value}}
    =
    \sum_{t,i}
    \Bigl(
        V_\phi^i\bigl(u_i^t,\{\hat{a}_{j\mid i}^t\}_{j\in\mathcal{L}_i^t}\bigr)
        - y_i^t
    \Bigr)^2,
\end{equation}
where $y_i^t$ is a one-step bootstrapped target computed using a target critic.
The actor is updated with an advantage-weighted log-likelihood objective:
\begin{equation}
    \mathcal{L}_{\mathrm{policy}}
    =
    - \sum_{t,i}
    \log \pi_\theta(a_i^t \mid u_i^t)\,\hat{A}_i^t,
\end{equation}
with detached advantages:
\begin{equation}
    \hat{A}_i^t
    =
    y_i^t
    -
    V_\phi^i\bigl(
        u_i^t,\{\hat{a}_{j\mid i}^t\}_{j\in\mathcal{L}_i^t}
    \bigr).
\end{equation}

We minimize the aggregate objective:
\begin{equation}
    \mathcal{L}
    =
    \mathcal{L}_{\mathrm{policy}}
    +
    \lambda_V\,\mathcal{L}_{\mathrm{value}}
    +
    \lambda_{\mathrm{topo}}\,\mathcal{L}_{\mathrm{topo}}
    +
    \lambda_{\mathrm{lead}}\,\mathcal{L}_{\mathrm{lead}},
\end{equation}
with non-negative weights
$\lambda_V,\lambda_{\mathrm{topo}},\lambda_{\mathrm{lead}}$.
Training proceeds by collecting joint trajectories with the current
decentralized policy and storing them in a replay buffer.
% the corresponding
% topological labels are constructed as in Sec.~\ref{subsec:formulation}. 
Mini-batches
sampled from the buffer are used to minimize $\mathcal{L}$, jointly updating the
priority inference modules and the topology-conditioned actor-critic.

\section{Experiments}

\subsection{Evaluation Platform}
\label{subsec:exp_platform}

\begin{figure*}[t]
    \centering
    \setlength{\belowcaptionskip}{-0.5cm}
    \includegraphics[width=6.5in]{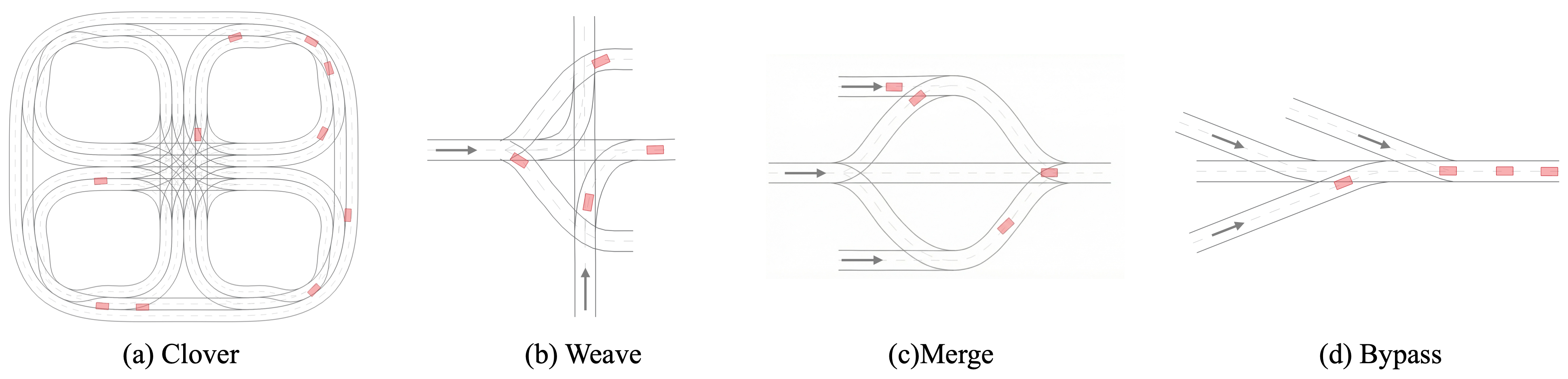}
    \caption{Evaluation scenarios for dense multi-vehicle interaction.}
  \label{fig:scenes}
\end{figure*}

We build our multi-agent traffic environment on \emph{VMAS}~\cite{bettini2022vmas}
and follow the environment design and evaluation protocol of
SigmaRL~\cite{Xu2024SigmaRL}. The simulator enables efficient parallel rollouts
with low-level continuous control for all vehicles. Fig.~\ref{fig:scenes} shows
representative maps. We evaluate on four traffic scenarios:

\textbf{Clover:} A cloverleaf interchange with repeated merge-diverge segments
and frequent short-range conflicts (initialize 20 vehicles).

\textbf{Weave:} A compact skewed intersection with tightly coupled turning and
crossing flows (initialize 8 vehicles).

\textbf{Merge:} An on-ramp merging task requiring gap acceptance and cooperative
speed adjustment (initialize 8 vehicles).

\textbf{Bypass:} A looped bypass connecting parallel road segments via an
alternative curved path (initialize 8 vehicles).

We report three evaluation metrics: collision rate (CR), average
speed (AS), and smoothness (SM). CR is decomposed into agent-agent and
agent-map events. Let $T$ be the episode length and $\mathbb{I}[\cdot]$ the
indicator:
\begin{equation}
\begin{aligned}
& \mathrm{CR}_{\mathrm{AA}} \triangleq \frac{1}{T}\sum_{t=1}^{T}\mathbb{I}\Big[\exists i:\ \text{collide}^{\,t}_{i,\mathrm{agent}}\Big]\times 100,\quad \\
& \mathrm{CR}_{\mathrm{AM}} \triangleq \frac{1}{T}\sum_{t=1}^{T}\mathbb{I}\Big[\exists i:\ \text{collide}^{\,t}_{i,\mathrm{map}}\Big]\times 100,
\end{aligned}
\end{equation}
and $\mathrm{CR}=\mathrm{CR}_{\mathrm{AA}}+\mathrm{CR}_{\mathrm{AM}}$.

AS measures traffic efficiency as the mean speed magnitude over agents and time steps,
normalized by $v_{\max}$:
\begin{equation}
\mathrm{AS} \triangleq \frac{1}{NT}\sum_{i=1}^{N}\sum_{t=1}^{T}\frac{\lVert v_i^t\rVert}{v_{\max}}\times 100.
\end{equation}

SM measures control variability via step-to-step changes in longitudinal and
steering commands:
\begin{equation}
\begin{aligned}
\mathrm{SM} \triangleq \frac{1}{N(T-1)} &\sum_{i=1}^{N}\sum_{t=2}^{T}
\beta\,\frac{\left|c^{t}_{i,\mathrm{lon}}-c^{t-1}_{i,\mathrm{lon}}\right|}{c_{\mathrm{lon,max}}}
+ \\
& (1-\beta)\,\frac{\left|c^{t}_{i,\mathrm{steer}}-c^{t-1}_{i,\mathrm{steer}}\right|}{c_{\mathrm{steer,max}}}
\times 100,
\end{aligned}
\end{equation}
where the first and second terms define the longitudinal and lateral components
$\mathrm{SM}_{\mathrm{LO}}$ and $\mathrm{SM}_{\mathrm{LA}}$, respectively.
Here $c_{\mathrm{lon,max}}$ and $c_{\mathrm{steer,max}}$ normalize the command
ranges and we set $\beta=0.5$.

\begin{table}[t]
\centering
\caption{Testing results on four scenarios from different MARL methods. 
Collision rate (CR$[\%]$) measures safety and is decomposed into agent-agent collisions $\mathrm{CR}_{\mathrm{AA}}$ and agent-map collisions $\mathrm{CR}_{\mathrm{AM}}$. Average speed (AS$[\%]$) measures traffic efficiency. Smoothness (SM$[\%]$) measures control stability.
Best is bold and second best is underlined within each scenario.
}
\label{tab:detailed_results}
\setlength{\tabcolsep}{6pt}
\renewcommand{\arraystretch}{1.0}
\small
% \begin{tabular}{c l r r r r}
\begin{tabular}{c l *{4}{>{\centering\arraybackslash}p{1.15cm}}}
\toprule
& & \multicolumn{1}{c}{MFPO} & \multicolumn{1}{c}{SigmaRL} & \multicolumn{1}{c}{XP-MARL} & \multicolumn{1}{c}{TSC} \\
\midrule

\multirow{9}{*}{\rotatebox[origin=c]{90}{Clover}}
& $\mathrm{CR}_{\mathrm{AA}}\downarrow$ & 1.71 & 1.67 & \underline{0.19} & \textbf{0.04} \\
& $\mathrm{CR}_{\mathrm{AM}}\downarrow$ & \textbf{0.00} & \textbf{0.00} & \underline{0.03} & \underline{0.03} \\
& $\mathrm{CR}\downarrow$              & 1.71 & 1.67 & \underline{0.22} & \textbf{0.07} \\
\cmidrule(lr){2-6}
& $\mathrm{AS}\uparrow$                & 86.6 & \underline{87.1} & 84.8 & \textbf{88.2} \\
\cmidrule(lr){2-6}
& $\mathrm{SM}_{\mathrm{LO}}\downarrow$& 2.63 & 2.58 & \underline{2.27} & \textbf{1.95} \\
& $\mathrm{SM}_{\mathrm{LA}}\downarrow$& 11.2 & 10.4 & \underline{8.95} & \textbf{8.89} \\
& $\mathrm{SM}\downarrow$              & 6.92 & 6.50 & \underline{5.61} & \textbf{5.42} \\
\midrule

\multirow{9}{*}{\rotatebox[origin=c]{90}{Weave}}
& $\mathrm{CR}_{\mathrm{AA}}\downarrow$ & 5.46 & 5.45 & \underline{0.69} & \textbf{0.18} \\
& $\mathrm{CR}_{\mathrm{AM}}\downarrow$ & 2.29 & \underline{0.95} & 0.98 & \textbf{0.81} \\
& $\mathrm{CR}\downarrow$              & 7.75 & 6.40 & \underline{1.66} & \textbf{0.99} \\
\cmidrule(lr){2-6}
& $\mathrm{AS}\uparrow$                & 72.5 & \underline{79.2} & 78.2 & \textbf{81.8} \\
\cmidrule(lr){2-6}
& $\mathrm{SM}_{\mathrm{LO}}\downarrow$& 9.44 & 7.59 & \underline{5.58} & \textbf{5.29} \\
& $\mathrm{SM}_{\mathrm{LA}}\downarrow$& 16.5 & 14.9 & \underline{9.51} & \textbf{9.38} \\
& $\mathrm{SM}\downarrow$              & 12.9 & 11.3 & \underline{7.54} & \textbf{7.33} \\
\midrule

\multirow{9}{*}{\rotatebox[origin=c]{90}{Merge}}
& $\mathrm{CR}_{\mathrm{AA}}\downarrow$ & 4.37 & 4.28 & \underline{0.24} & \textbf{0.03} \\
& $\mathrm{CR}_{\mathrm{AM}}\downarrow$ & \textbf{0.00} & \textbf{0.00} & \textbf{0.00} & \textbf{0.00} \\
& $\mathrm{CR}\downarrow$              & 4.37 & 4.28 & \underline{0.24} & \textbf{0.03} \\
\cmidrule(lr){2-6}
& $\mathrm{AS}\uparrow$                & 78.6 & \underline{83.1} & 81.4 & \textbf{84.7} \\
\cmidrule(lr){2-6}
& $\mathrm{SM}_{\mathrm{LO}}\downarrow$& 5.85 & \underline{5.23} & 5.57 & \textbf{5.09} \\
& $\mathrm{SM}_{\mathrm{LA}}\downarrow$& 10.0 & 8.56 & \textbf{5.32} & \underline{6.58} \\
& $\mathrm{SM}\downarrow$              & 7.95 & 6.90 & \textbf{5.45} & \underline{5.84} \\
\midrule

\multirow{9}{*}{\rotatebox[origin=c]{90}{Bypass}}
& $\mathrm{CR}_{\mathrm{AA}}\downarrow$ & 4.80 & 5.02 & \underline{0.56} & \textbf{0.08} \\
& $\mathrm{CR}_{\mathrm{AM}}\downarrow$ & 0.78 & \textbf{0.03} & \underline{0.07} & \textbf{0.03} \\
& $\mathrm{CR}\downarrow$              & 5.57 & 5.06 & \underline{0.62} & \textbf{0.10} \\
\cmidrule(lr){2-6}
& $\mathrm{AS}\uparrow$                & 76.4 & \textbf{81.4} & 78.0 & \underline{79.4} \\
\cmidrule(lr){2-6}
& $\mathrm{SM}_{\mathrm{LO}}\downarrow$& 6.93 & \underline{6.18} & \textbf{5.94} & 7.44 \\
& $\mathrm{SM}_{\mathrm{LA}}\downarrow$& 14.5 & 14.4 & \textbf{9.02} & \underline{9.16} \\
& $\mathrm{SM}\downarrow$              & 10.7 & 10.3 & \textbf{7.48} & \underline{8.30} \\
\bottomrule
\end{tabular}
\end{table}

\subsection{Experiment Setting}
\label{subsec:exp_setting}

We compare TSC with three baselines: MFPO, SigmaRL, and XP-MARL. MFPO uses a mean-field
aggregation of nearby vehicles instead of explicit pairwise interactions, enabling scalable learning and execution as the number of agents grows \cite{yang2018mean,nayak2023scalable}. 
SigmaRL is a MARL approach that emphasizes sample-efficient and generalizable value learning for multi-vehicle motion planning \cite{Xu2024SigmaRL}.
XP-MARL extends SigmaRL with sequential coordination through auxiliary prioritization \cite{Xu2024XPMARL}. 
XP-MARL relies on a centralized controller to assign priorities and propagate higher-priority actions, whereas our method infers priorities locally under decentralized, communication-free execution. 

All methods are trained and evaluated in VMAS on the four environments. 
Each episode consists of a horizon of $1200$ steps with $\Delta t=0.05\,\mathrm{s}$. 
Under partial observability we
set $|\mathcal{N}_i^t|=4$ for priority inference and $|\mathcal{N}_{i,K}^t|=2$
for policy learning and execution. Each run is trained for $250$ iterations and
collects $4096$ environment steps per iteration, which yields about
$1.0\times10^{6}$ total steps. We set $\lambda_{\mathrm{node}}=1.0$ and
$\lambda_{\mathrm{cons}}=1.0$.
All experiments are conducted on an AMD Ryzen 9 7945HX CPU with 32\,GB RAM.

\subsection{Main Results}
\label{subsec:result}

\begin{figure*}[t]
    \centering
    \includegraphics[width=6.5in]{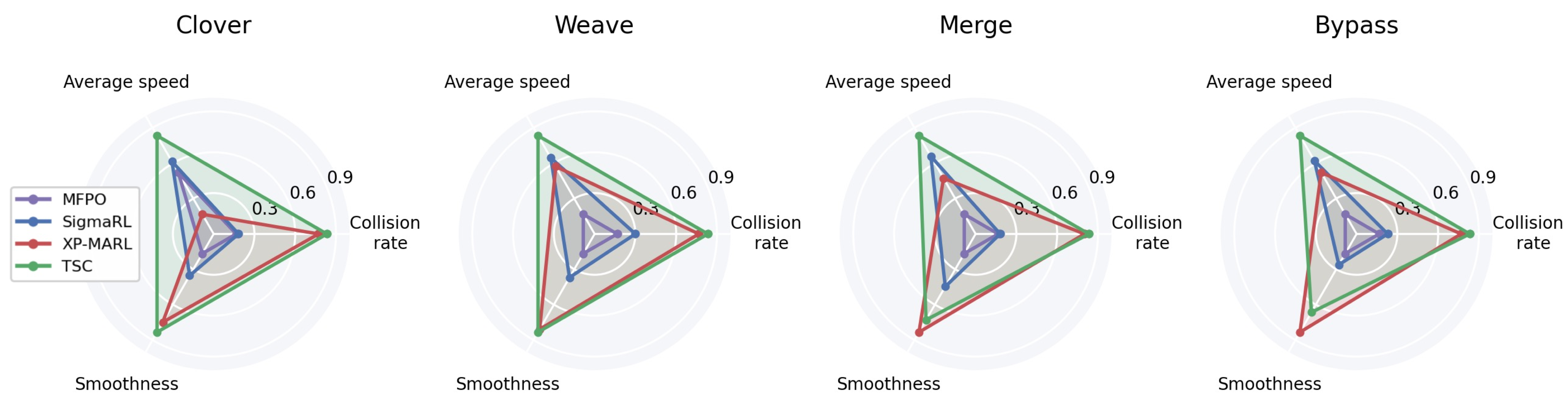}
    \caption{Comparison of safety, efficiency, and smoothness across four scenarios.}
  \label{fig:result1}
\end{figure*}

\begin{figure*}[t]
    \centering
    \includegraphics[width=6.5in]{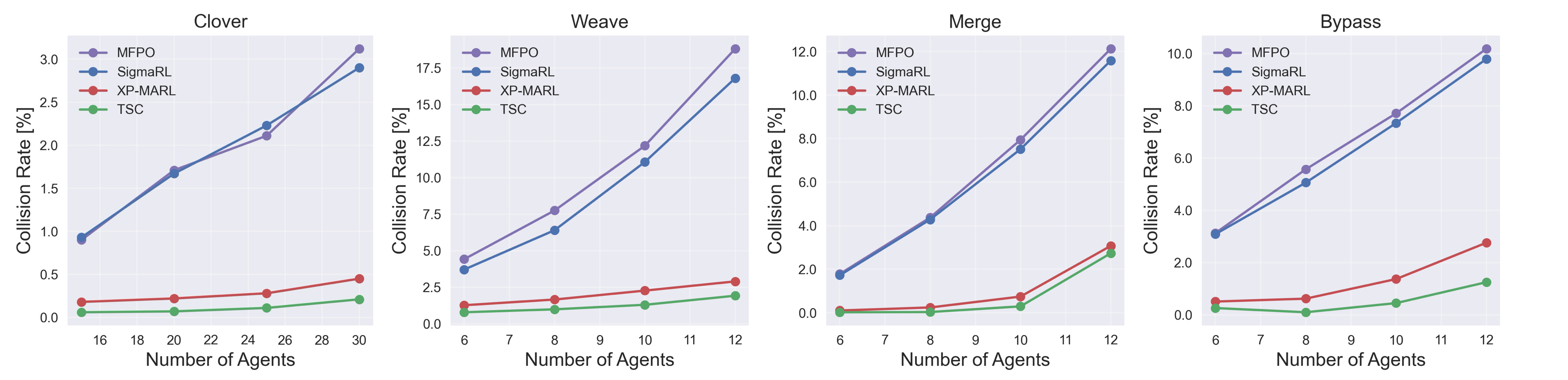}
    \caption{Collision rate of different methods under varying numbers of vehicles at test time.}
  \label{fig:result2}
\end{figure*}

\begin{figure*}[t]
    \centering
    \setlength{\belowcaptionskip}{-0.5cm}
    \includegraphics[width=6.6in]{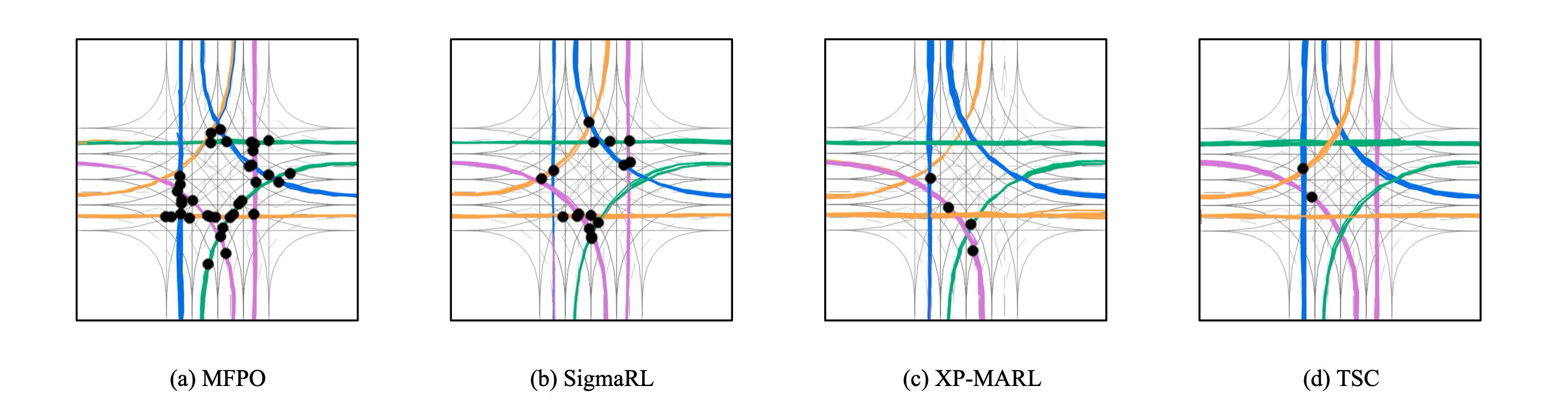}
    \caption{Visualization of vehicle trajectories in one test episode. Black dots indicate collision events. 
    }
  \label{fig:traj_vis}
\end{figure*}

Table~\ref{tab:detailed_results} summarizes performance on four scenarios. TSC achieves the lowest total collision rate in all scenarios while maintaining competitive average speed, which indicates that the safety improvement is not obtained by conservative slowdowns. The gain is most pronounced in \textit{Weave}, where TSC attains a collision rate of 0.99\% with the highest AS of 81.8. In comparison, MFPO and SigmaRL exhibit substantially higher collision rates in the same scenario, at 7.75\% and 6.40\% respectively. Similar safety improvements appear in \textit{Clover}, \textit{Merge}, and \textit{Bypass}, where TSC consistently reduces collisions to near-zero levels and preserves traffic efficiency.

Decomposed results show that the reduction is mainly driven by fewer agent-agent collisions. This is consistent with our design that infers a directed local priority graph and uses it to structure interactions under decentralized execution. For example, in \textit{Weave} the agent-agent collision rate drops to 0.18\% under TSC, compared with 5.46\% under MFPO and 5.45\% under SigmaRL. XP-MARL also improves safety through sequential coordination, but it relies on centralized priority management. TSC attains lower total collision rates across all scenarios under communication-free execution, suggesting that locally inferred priorities can replace centralized ordering in dense traffic. TSC also achieves the best or second best smoothness in most cases, with the main exceptions on \textit{Merge} and \textit{Bypass}, where XP-MARL yields smoother control on $SM_{LA}$ or $SM_{LO}$.

\begin{figure*}[t]
    \centering
    \includegraphics[width=6.5in]{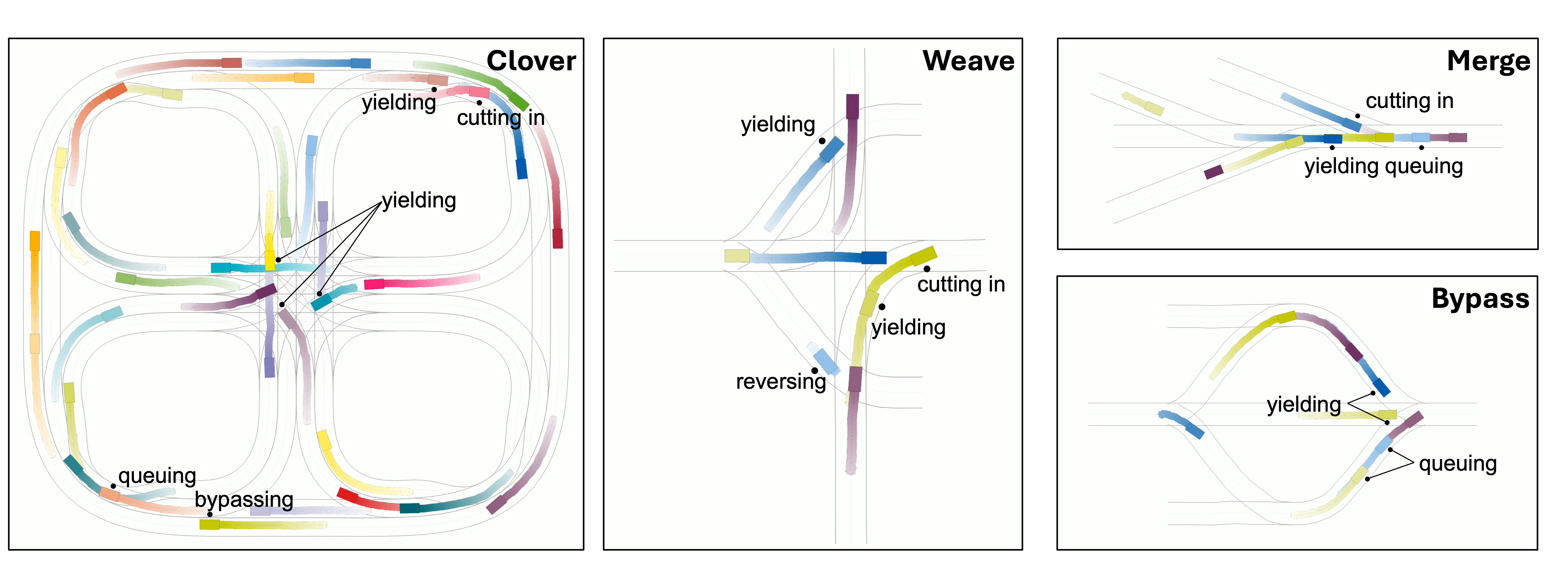}
    \caption{Qualitative trajectory visualization of TSC across four interaction scenarios, highlighting emergent behaviors including yielding, queuing, cutting in, bypassing, and reversing.
    For each vehicle, we plot a 25-step trajectory segment with distinct colors, where the luminosity fades from light to deep to indicate earlier to later positions.}
  \label{fig:behavior}
\end{figure*}

Fig.~\ref{fig:result1} presents radar plots over collision rate, average speed, and smoothness. For each scenario, we normalize each metric to its maximal and minimal values over all evaluated episodes and all methods, and use the negative collision rate for safety normalization. The plots show that XP-MARL and TSC both expand the safety axis, highlighting the benefit of sequential Stackelberg-style coordination, whereas TSC yields a more balanced profile across all axes due to topology-guided sparsification and graph-local leader--follower reasoning, which stabilizes interactions and avoids over-conditioning on centralized action propagation.
% decentralized, communication-free priority inference.

To further evaluate the generalization under varying traffic density, we execute the learned policies in the same test scenarios while changing the number of vehicles. As shown in Fig.~\ref{fig:result2}, MFPO and SigmaRL are more sensitive to the number of vehicles, with collision rates increasing rapidly as the traffic becomes denser. In contrast, XP-MARL and TSC maintain substantially lower collision rates across all tested densities. 
Notably, TSC consistently yields the lowest collision rates and exhibits the slowest degradation as the number of vehicles increases, suggesting that the topology-conditioned Stackelberg decision structure enables scalable, stable coordination by converting dense couplings into sparse leader-follower dependencies.

\subsection{Behavior Analysis}
\label{subsec:analysis}

In Fig.~\ref{fig:traj_vis}, MFPO and SigmaRL show dense trajectory crossings and collision clusters around the conflict region, which indicates unresolved right-of-way interactions under decentralized control. XP-MARL reduces collisions by enforcing sequential reactions, yet residual collision markers remain near merging and crossing areas. TSC produces clearer precedence patterns with better separated traffic streams and only sparse collision events, illustrating how topology-conditioned priority inference can stabilize local leader follower responses without communication.

Across four scenarios, Fig.~\ref{fig:behavior} highlights recurrent interaction patterns induced by TSC, including yielding and queuing near conflict points, and proactive gap creation behaviors such as cutting in, bypassing, and occasional reversing. These behaviors arise from local observations with a shared policy, showing that the inferred priority graph provides a compact coordination signal that transfers across different geometries and traffic densities.

\begin{table}[htbp]
  \centering
  \caption{Ablation study on total collision rate  (CR$[\%]$) across different scenarios.}
  \label{tab:ablation_collision}
  \setlength{\tabcolsep}{6pt}
    \renewcommand{\arraystretch}{1.2}
    \small
  \begin{tabular}{lcccc}
    \toprule
    \textbf{Method} & \textbf{Clover} & \textbf{Weave} & \textbf{Merge} & \textbf{Bypass} \\
    \midrule
    TSC w/ random priority & 0.30 & 1.62 & 1.11 & 1.79 \\
    TSC w/o Stackelberg & 0.27 & 1.08 & 0.78 & 1.91 \\
    TSC w/o Top-$K$ Filter & 1.00 & 2.92 & 0.97 & 1.87 \\
    \hline
    \textbf{TSC} & \textbf{0.07} & \textbf{0.99} & \textbf{0.03} & \textbf{0.10} \\
    \bottomrule
  \end{tabular}
\end{table}

\subsection{Ablation Studies}

We conduct ablation studies to verify the effectiveness of the proposed coordination design. Table~\ref{tab:ablation_collision} compares TSC with three variants. “TSC with random priority” replaces the inferred topological priorities with random pairwise priorities, which consistently increases collisions, especially in \textit{Weave} and \textit{Bypass}. This indicates that the priority structure must reflect the interaction topology rather than being arbitrary. “TSC without Stackelberg” removes the leader–follower conditioning and yields higher collision rates in most scenarios, suggesting that the Stackelberg-style dependency helps stabilize interactive decisions under non-stationarity. “TSC without Top-K filter” disables topological sparse selection and considers all neighbors, leading to the worst safety outcomes across scenarios, showing that sparse selection is not only a computational choice but also a safety-critical mechanism that suppresses irrelevant interactions.

Fig.~\ref{fig:collision-curve} further studies the influence of  neighborhood size for sparse selection on agent-agent collision rate. We compare TopoNbr-1/2/3/5, where the policy conditions on the topologically most relevant K neighbors.
As illustrated in Fig.~\ref{fig:collision-curve}, TopoNbr-2 achieves the lowest collision rate and the fastest reduction during training, indicating improved stability and sample efficiency. TopoNbr-1 tends to miss secondary conflict partners, While larger K includes weakly coupled neighbors, which dilutes the influence of the truly critical vehicles and injects inconsistent interaction cues, resulting in less stable learning and more agent–agent collisions. Overall, the results confirm that meaningful topological priorities, Stackelberg-style conditioning, and a small but informative neighborhood jointly contribute to safer coordination.

\begin{figure}[t]
    \centering
    \setlength{\belowcaptionskip}{-0.5cm}
    \includegraphics[width=3.2in]{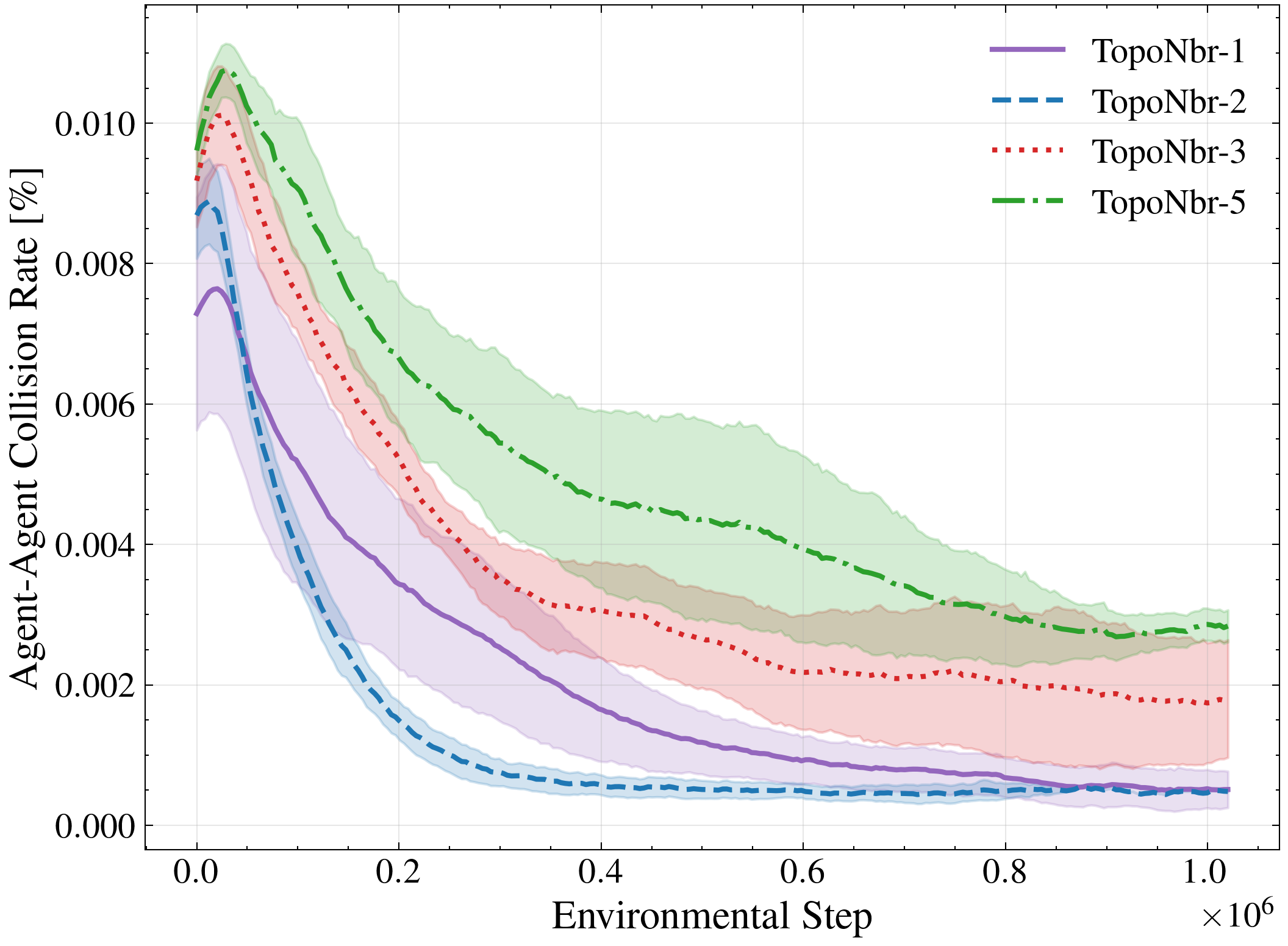}
    \caption{Comparison of TSC with different TopoNbr-K settings.}
  \label{fig:collision-curve}
\end{figure}

\section{Conclusion}

This paper introduced TSC, a topology-conditioned Stackelberg coordination framework for resolving spatiotemporal conflicts among decentralized autonomous agents in dense traffic. A weaving-induced directed priority graph encodes asymmetric local precedence, and its edges are interpreted as leader–follower constraints that couple local Stackelberg subgames without constructing a global order of play. Based on this formulation, TSC-Net infers ego-centric priority structures from local observations and leverages predicted leader behaviors to train a Stackelberg-conditioned actor–critic under CTDE, enabling communication-free execution.

Empirical results across representative dense traffic scenarios demonstrate substantial collision reduction while maintaining competitive system efficiency and smooth control, and show slower performance degradation under increased agent densities. 
% Real-world deployment further supports that the learned priorities induce clear yielding and passing behaviors without explicit communication. 
Future work will consider richer heterogeneity such as mixed dynamics and intent classes, improved robustness to perception uncertainty and occlusion, and the integration of explicit safety constraints or formal guarantees.

\bibliographystyle{IEEEtran}
\bibliography{TSC}

\appendix

\subsection{Bellman error bound under leader-prediction noise}

Fix a follower agent $i$ and its local leader set $\mathcal{L}_i^t$ defined in Sec.~\ref{subsubsec:stackelberg-coordination}.
Let $T_i^{\mathrm{SE}}$ denote the Bellman evaluation operator of the local Stackelberg MDP in \eqref{eq:se-bellman}, where leader actions are the realized actions $a_j^t\sim\pi_\theta(\cdot\mid u_j^t)$ for all $j\in\mathcal{L}_i^t$.
Let $T_i^{\mathrm{SE,pred}}$ denote an auxiliary operator obtained by replacing these leader actions in the Bellman backup by the corresponding predictions $\hat{a}_{j\mid i}^t$ for all $j\in\mathcal{L}_i^t$.
Let $V_i^{\mathrm{SE}}$ and $V_i^{\mathrm{SE,pred}}$ be the unique fixed points of $T_i^{\mathrm{SE}}$ and $T_i^{\mathrm{SE,pred}}$, respectively.

\begin{lemma}[Bellman error bound]
\label{lem:se-bellman-bound}
Assume the per-step reward is bounded in magnitude by $|r_i^t|\le R_{\max}$, so any discounted value function under these operators satisfies $\|V\|_\infty\le V_{\max}$ with $V_{\max}=R_{\max}/(1-\gamma)$.
Suppose that for some $\varepsilon_T \ge 0$ we have
\begin{equation}
    \bigl\|
        T_i^{\mathrm{SE}}V - T_i^{\mathrm{SE,pred}}V
    \bigr\|_\infty
    \le \varepsilon_T
\end{equation}
for all bounded $V$ with $\|V\|_\infty\le V_{\max}$.
Then
\begin{equation}
    \bigl\|
        V_i^{\mathrm{SE}} - V_i^{\mathrm{SE,pred}}
    \bigr\|_\infty
    \le
    \frac{\varepsilon_T}{1-\gamma}.
    \label{eq:bellman-bound}
\end{equation}
\end{lemma}

\begin{proof}
Recall that both $T_i^{\mathrm{SE}}$ and $T_i^{\mathrm{SE,pred}}$ are
$\gamma$-contractions in the sup norm.
Using $V_i^{\mathrm{SE}} = T_i^{\mathrm{SE}}V_i^{\mathrm{SE}}$ and
$V_i^{\mathrm{SE,pred}} = T_i^{\mathrm{SE,pred}}V_i^{\mathrm{SE,pred}}$, we
obtain
\begin{equation}
\begin{aligned}
\bigl\|
    V_i^{\mathrm{SE}} - V_i^{\mathrm{SE,pred}}
\bigr\|_\infty
&=
\bigl\|
    T_i^{\mathrm{SE}}V_i^{\mathrm{SE}}
    -
    T_i^{\mathrm{SE,pred}}V_i^{\mathrm{SE,pred}}
\bigr\|_\infty \\
&\le
\bigl\|
    T_i^{\mathrm{SE}}V_i^{\mathrm{SE}}
    -
    T_i^{\mathrm{SE,pred}}V_i^{\mathrm{SE}}
\bigr\|_\infty \\
& +
\bigl\|
    T_i^{\mathrm{SE,pred}}V_i^{\mathrm{SE}}
    -
    T_i^{\mathrm{SE,pred}}V_i^{\mathrm{SE,pred}}
\bigr\|_\infty \\
&\le
\varepsilon_T
+
\gamma
\bigl\|
    V_i^{\mathrm{SE}} - V_i^{\mathrm{SE,pred}}
\bigr\|_\infty,
\end{aligned}
\end{equation}
where the first term is bounded by assumption and the second term uses
the contraction property of $T_i^{\mathrm{SE,pred}}$.
Rearranging gives
$
(1-\gamma)
\|V_i^{\mathrm{SE}} - V_i^{\mathrm{SE,pred}}\|_\infty
\le \varepsilon_T
$,
which yields \eqref{eq:bellman-bound}.
\end{proof}

Lemma~\ref{lem:se-bellman-bound} shows that if the Bellman backup error
caused by leader-action prediction is uniformly small, then the induced
value $V_i^{\mathrm{SE,pred}}$ is uniformly close to the ideal local
Stackelberg value $V_i^{\mathrm{SE}}$, with error amplified by at most
the factor $1/(1-\gamma)$.

\subsection{Policy improvement in the local Stackelberg MDP}

For a follower policy $\pi_i$ and fixed leader policies
$\pi_{\mathcal{L}_i^t}=\{\pi_j\}_{j\in\mathcal{L}_i^t}$,
the local Stackelberg follower objective is
\begin{equation}
\begin{aligned}
    J_i^{\mathrm{SE}}(\pi_i;\pi_{\mathcal{L}_i^t})
    = &
    \mathbb{E}\Bigl[
        \sum_{k=0}^{\infty}\gamma^k r_i^{t+k}
        \,\Bigm|\,
        a_i^{t+k} \sim \pi_i(\cdot\mid u_i^{t+k}),\ \\
        & a_j^{t+k} \sim \pi_j(\cdot\mid u_j^{t+k}),\ \forall j\in\mathcal{L}_i^t
    \Bigr],
    \label{eq:local-se-objective}
\end{aligned}
\end{equation}
where the expectation is taken over trajectories induced by
$(\pi_i,\pi_{\mathcal{L}_i^t})$.
Let $V_{i,\pi}^{\mathrm{SE}}$ and $Q_{i,\pi}^{\mathrm{SE}}$ denote the
corresponding value and action value functions, and define the advantage
\begin{equation}
    A_{i,\pi}^{\mathrm{SE}}(u,a)
    =
    Q_{i,\pi}^{\mathrm{SE}}(u,a)
    -
    V_{i,\pi}^{\mathrm{SE}}(u).
\end{equation}

\begin{lemma}[Performance difference lemma for the local Stackelberg MDP]
\label{lem:se-pdl}
Consider the local Stackelberg MDP of agent $i$ with fixed leader policies $\pi_{\mathcal{L}_i^t} = \{\pi_j\}_{j\in\mathcal{L}_i^t}$ at a given time step $t$ and its associated leader set $\mathcal{L}_i^t$.
For any two follower policies $\pi_i,\tilde\pi_i$ and the corresponding discounted state visitation distribution $d_{\pi_i}$ induced by $(\pi_i,\pi_{\mathcal{L}_i^t})$, we have
\begin{equation}
\begin{aligned}
    J_i^{\mathrm{SE}}(\pi_i;\pi_{\mathcal{L}_i^t})
    -
    & J_i^{\mathrm{SE}}(\tilde\pi_i;\pi_{\mathcal{L}_i^t})
    = \\
    & \frac{1}{1-\gamma}\,
    \mathbb{E}_{u\sim d_{\pi_i},\,a\sim\pi_i(\cdot\mid u)}
    \bigl[
        A_{i,\tilde\pi_i}^{\mathrm{SE}}(u,a)
    \bigr],
    \label{eq:app-se-pdl-action}
\end{aligned}
\end{equation}
where
\begin{equation}
    d_{\pi_i}(u)
    =
    (1-\gamma)\sum_{k=0}^{\infty}
    \gamma^k\,\Pr\bigl(u_i^{t+k} = u \mid \pi_i,\pi_{\mathcal{L}_i^t}\bigr).
\end{equation}

If $\pi_i$ is deterministic, $a=\pi_i(u)$, we may write
$A_{i,\tilde\pi_i}^{\mathrm{SE}}(u)
 := A_{i,\tilde\pi_i}^{\mathrm{SE}}(u,\pi_i(u))$ and obtain
\begin{equation}
    J_i^{\mathrm{SE}}(\pi_i;\pi_{\mathcal{L}_i^t})
    -
    J_i^{\mathrm{SE}}(\tilde\pi_i;\pi_{\mathcal{L}_i^t})
    =
    \frac{1}{1-\gamma}\,
    \mathbb{E}_{u\sim d_{\pi_i}}
    \bigl[
        A_{i,\tilde\pi_i}^{\mathrm{SE}}(u)
    \bigr].
    \label{eq:app-se-pdl-state}
\end{equation}
\end{lemma}

\begin{proof}
Fix the leader policies $\pi_{\mathcal{L}_i^t}$.
Let $V_{i,\tilde\pi_i}^{\mathrm{SE}}$ be the value function under
$(\tilde\pi_i,\pi_{\mathcal{L}_i^t})$.
For any state $u$ and action $a$, define
\begin{equation}
    Q_{i,\tilde\pi_i}^{\mathrm{SE}}(u,a)
    =
    \mathbb{E}\Bigl[
        r_i^{t}
        +
        \gamma
        V_{i,\tilde\pi_i}^{\mathrm{SE}}(u_i^{t+1})
        \Bigm|
        u_i^{t}=u,a_i^{t}=a,\pi_{\mathcal{L}_i^t}
    \Bigr].
\end{equation}
By the definition of advantage,

\begin{equation}
\begin{aligned}
    A_{i,\tilde\pi_i}^{\mathrm{SE}}(u,a)
    & = 
    \mathbb{E}\Bigl[
        r_i^{t}
        +
        \gamma
        V_{i,\tilde\pi_i}^{\mathrm{SE}}(u_i^{t+1})
        \\ & -
        V_{i,\tilde\pi_i}^{\mathrm{SE}}(u_i^{t})
        \,\Bigm|\,
        u_i^{t}=u,\ a_i^{t}=a,\ \pi_{\mathcal{L}_i^t}
    \Bigr].
\end{aligned}
\end{equation}

Now consider trajectories generated by $(\pi_i,\pi_{\mathcal{L}_i^t})$
starting from time $t$.
Taking conditional expectation of the previous identity at each time
$t+k$ and summing with weights $\gamma^k$ gives
\begin{equation}
\begin{aligned}
\mathbb{E}\Bigl[
\sum_{k=0}^{\infty}\gamma^k
A_{i,\tilde\pi_i}^{\mathrm{SE}}(u_i^{t+k},a_i^{t+k})
\Bigr]
= & \mathbb{E}\Bigl[
\sum_{k=0}^{\infty}\gamma^k r_i^{t+k}
+ \\
\sum_{k=0}^{\infty}\gamma^{k+1}V_{i,\tilde\pi_i}^{\mathrm{SE}}(u_i^{t+k+1})
- &
\sum_{k=0}^{\infty}\gamma^{k}V_{i,\tilde\pi_i}^{\mathrm{SE}}(u_i^{t+k})
\Bigr].
\end{aligned}
\end{equation}

The two value sums telescope, leaving
\begin{equation}
\mathbb{E}\Bigl[
\sum_{k=0}^{\infty}\gamma^k
A_{i,\tilde\pi_i}^{\mathrm{SE}}(u_i^{t+k},a_i^{t+k})
\Bigr]
=
J_i^{\mathrm{SE}}(\pi_i;\pi_{\mathcal{L}_i^t})
-
J_i^{\mathrm{SE}}(\tilde\pi_i;\pi_{\mathcal{L}_i^t}),
\end{equation}
where the second term appears because under
$(\tilde\pi_i,\pi_{\mathcal{L}_i^t})$,
$V_{i,\tilde\pi_i}^{\mathrm{SE}}(u_i^{t})$ equals the discounted return
starting from state $u_i^{t}$.

Finally, rewrite the left-hand side using the discounted visitation
distribution.
By definition of $d_{\pi_i}$,
\begin{equation}
\begin{aligned}
\mathbb{E}\Bigl[
\sum_{k=0}^{\infty}\gamma^k
A_{i,\tilde\pi_i}^{\mathrm{SE}} &(u_i^{t+k},a_i^{t+k})
\Bigr]
= \\
& \frac{1}{1-\gamma}
\mathbb{E}_{u\sim d_{\pi_i},\,a\sim\pi_i(\cdot\mid u)}
\bigl[
A_{i,\tilde\pi_i}^{\mathrm{SE}}(u,a)
\bigr],
\end{aligned}
\end{equation}
which yields \eqref{eq:app-se-pdl-action}.
The deterministic case \eqref{eq:app-se-pdl-state} follows by substituting
$a=\pi_i(u)$.
\end{proof}

% \begin{corollary}[]
% % Fix the leader policies $\pi_{\mathcal{L}_i^t} = \{\pi_j\}_{j\in\mathcal{L}_i^t}$ and consider the local Stackelberg MDP of follower $i$.
% Let $V_\phi^i\bigl(u_i^t,\{\hat{a}_{j\mid i}^t\}_{j\in\mathcal{L}_i^t}\bigr)$ denote the Stackelberg-conditioned critic in~\eqref{eq:valhead}.
% If $V_\phi^i$ matches the local Stackelberg state-value $V_i^{\mathrm{SE}}(u_i^t)$ on the support of the discounted visitation distribution $d_{\pi_i}$, then the policy-gradient update based on advantages estimated from $V_\phi^i$ is a stochastic gradient-ascent step on the follower objective $J_i^{\mathrm{SE}}(\pi_i;\pi_{\mathcal{L}_i^t})$ in~\eqref{eq:local-se-objective}.
% In particular, for fixed higher-priority neighbors $\mathcal{L}_i^t$, the actor update with the Stackelberg-conditioned critic realizes a local Stackelberg follower best response, while the evolving leader policies $\pi_{\mathcal{L}_i^t}$ implicitly specify the upper level of a decentralized bilevel Stackelberg optimization.
% \end{corollary}

\begin{corollary}[]
Let $V_\phi^i\bigl(u_i^t,\{\hat{a}_{j\mid i}^t\}_{j\in\mathcal{L}_i^t}\bigr)$ denote the Stackelberg-conditioned critic in~\eqref{eq:valhead}.
If $V_\phi^i$ matches the predicted local Stackelberg state-value $V_i^{\mathrm{SE,pred}}(u_i^t)$ on the support of the discounted visitation distribution $d_{\pi_i}$, then the policy-gradient update based on advantages estimated from $V_\phi^i$ is a stochastic gradient-ascent step on the predicted follower objective $J_i^{\mathrm{SE,pred}}(\pi_i;\pi_{\mathcal{L}_i^t})$ induced by $T_i^{\mathrm{SE,pred}}$.
Moreover, when the induced bias is controlled as in Lemma~\ref{lem:se-bellman-bound}, this update provides a tractable approximation to improving the ideal follower objective $J_i^{\mathrm{SE}}(\pi_i;\pi_{\mathcal{L}_i^t})$.
In particular, for fixed higher-priority neighbors $\mathcal{L}_i^t$, the actor update with the Stackelberg-conditioned critic realizes a local Stackelberg follower best response, while the evolving leader policies $\pi_{\mathcal{L}_i^t}$ implicitly specify the upper level of a decentralized bilevel Stackelberg optimization.
\end{corollary}

\vfill

\end{document}